\documentclass{article}

\usepackage{microtype}
\usepackage{graphicx}
\usepackage{subfigure}
\usepackage{booktabs} % for professional tables
\usepackage{url}
\usepackage{times}
\usepackage{epsfig}
\usepackage{graphicx}
\usepackage{amsmath,amssymb,amsthm}
\usepackage{mathtools}
\usepackage{booktabs}
\usepackage{comment}
\usepackage{color}
\usepackage{colortbl}
\usepackage[nice]{nicefrac}
\usepackage{multirow}
\usepackage{caption}
\usepackage{subcaption}
\usepackage{wrapfig}
\usepackage{xspace}
\usepackage{adjustbox}
\usepackage{soul}
\usepackage{algorithm,algorithmicx,algpseudocode}
\usepackage{enumitem}

\usepackage{hyperref}

\usepackage[accepted]{icml2024}
\usepackage{amsmath,amsfonts,bm}

\usepackage{amsmath}
\usepackage{amssymb}
\usepackage{mathtools}
\usepackage{amsthm}

\usepackage[capitalize,noabbrev]{cleveref}

\theoremstyle{plain}
\newtheorem{theorem}{Theorem}[section]
\newtheorem{proposition}[theorem]{Proposition}

\theoremstyle{definition}

\theoremstyle{remark}

\usepackage[textsize=tiny]{todonotes}

\newcommand\numberthis{\addtocounter{equation}{1}\tag{\theequation}}
\newcommand{\codecomment}[1]{{\color{gray}{#1}}}
\definecolor{fbApp}{HTML}{c8e7fa}

\newcommand{\cc}{\cellcolor{fbApp}}

\definecolor{lightgrey}{rgb}{0.925, 0.925, 0.925}
\sethlcolor{lightgrey}

\newlength\savewidth\newcommand\shline{\noalign{\global\savewidth\arrayrulewidth
  \global\arrayrulewidth 1pt}\hline\noalign{\global\arrayrulewidth\savewidth}}
\newcommand{\tablestyle}[2]{\setlength{\tabcolsep}{#1}\renewcommand{\arraystretch}{#2}\centering\footnotesize}

\newcolumntype{x}[1]{>{\centering\arraybackslash}p{#1pt}}
\newcolumntype{y}[1]{>{\raggedright\arraybackslash}p{#1pt}}
\newcolumntype{z}[1]{>{\raggedleft\arraybackslash}p{#1pt}}

\usepackage{amsmath, amsthm, amssymb}

\theoremstyle{plain}

\icmltitlerunning{Stochastic positional embeddings improve masked image modeling}

\begin{document}

\twocolumn[

\icmltitle{Stochastic positional embeddings improve masked image modeling}

\begin{icmlauthorlist}
 
\icmlauthor{Amir Bar}{1,2,3}
\icmlauthor{Florian Bordes}{3}
\icmlauthor{Assaf Shocher}{2}
\icmlauthor{Mahmoud Assran}{3}
\icmlauthor{Pascal Vincent}{3}
\icmlauthor{Nicolas Ballas}{3}
\icmlauthor{Trevor Darrell}{2}
\icmlauthor{Amir Globerson}{1}
\icmlauthor{Yann LeCun}{3,4}
\end{icmlauthorlist}

\icmlaffiliation{1}{Tel Aviv University}
\icmlaffiliation{2}{UC Berkeley}
\icmlaffiliation{3}{Meta AI (FAIR)}
\icmlaffiliation{4}{New York University}

\icmlcorrespondingauthor{Amir Bar}{amir.bar@cs.tau.ac.il}

\icmlkeywords{Machine Learning, ICML}

\vskip 0.3in
]

\printAffiliationsAndNotice{}  % leave blank if no need to mention equal 

\begin{abstract}
Masked Image Modeling (MIM) is a promising self-supervised learning approach that enables learning from unlabeled images. Despite its recent success, learning good representations through MIM remains challenging because it requires predicting the right semantic content in accurate locations. For example, given an incomplete picture of a dog, we can guess that there is a tail, but we cannot determine its exact location. In this work, we propose to incorporate location uncertainty into MIM by using stochastic positional embeddings (StoP). Specifically, we condition the model on stochastic masked token positions drawn from a Gaussian distribution. StoP reduces overfitting to location features and guides the model toward learning features that are more robust to location uncertainties. Quantitatively, StoP improves downstream MIM performance on a variety of downstream tasks, including $+1.7\%$ on ImageNet linear probing using ViT-B, and $+2.5\%$ for ViT-H using 1\% of the data.\footnote{See~\url{https://github.com/amirbar/StoP} for code.}
\end{abstract}
\section{Introduction}
\begin{figure}[t]
    \centering
   \captionsetup{type=figure}
\includegraphics[width=1.\linewidth]{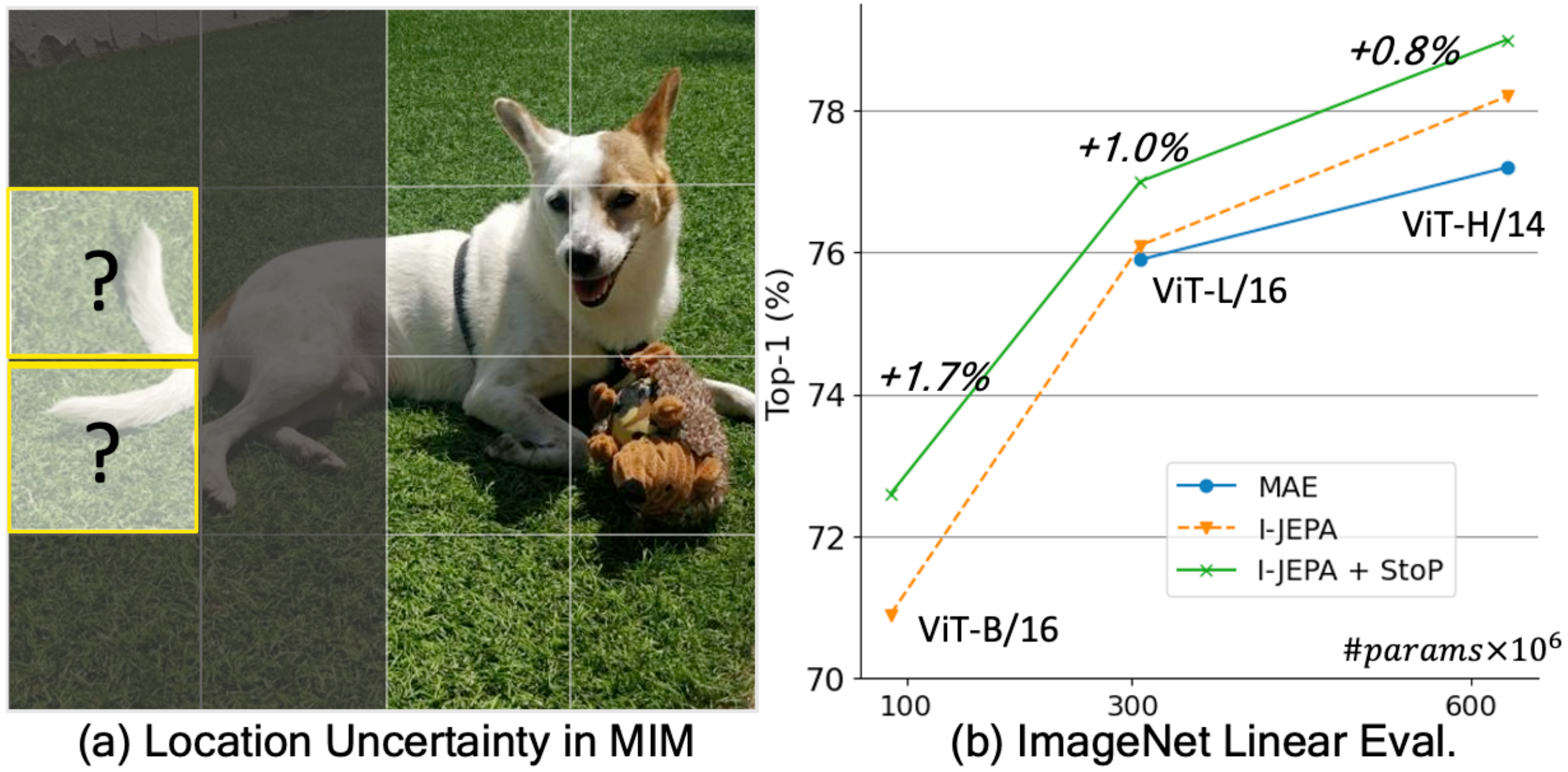}
    \captionof{figure}{\small Given a partial image of a dog, can you precisely determine the location of its tail? Existing Masked Image Modeling (MIM) models like MAE~\citep{he2021masked} and I-JEPA~\citep{assran2023self} predict tokens deterministically and do not model location uncertainties (a), we propose to predict the target (masked tokens) in stochastic positions (StoP) which prevents overfitting to locations features. StoP leads to improved MIM performance on downstream tasks, including linear probing on ImageNet (b).}
        \label{fig:teaser}
\end{figure}

Masked Image Modeling (MIM) enables learning from unlabeled images by reconstructing masked parts of the image given the rest of the image as context. In recently years, new MIM methods have emerged~\citep{xie2021simmim, bao2021beit, he2021masked, assran2023self}. Masked Auto-Encoders (MAE)~\citep{he2021masked} are trained to minimize a reconstruction error in pixel space, and I-JEPA~\citep{assran2023self} reconstructs image features. MIM is appealing compared to invariance-based self-supervised learning methods like DINO~\citep{caron2021emerging} and iBOT~\citep{zhou2021ibotyes} as MIM do not suffer from the same limitations, namely, it does not require heavy use of hand-crafted augmentations~\citep{xiaoshould,he2021masked}, mini-batch statistics, or a uniform cluster prior~\citep{assran2022hidden}.

Despite the recent success of MIM, we argue that learning good representations using MIM remains challenging due to location uncertainties because it requires predicting the right semantic content in accurate locations. For example, given an incomplete picture of a dog (see Figure~\ref{fig:teaser}a), we might guess there's a tail, but we can't be sure exactly where it is, as it could realistically be in several different places. Without explicitly modeling this location uncertainty, existing MIM models like MAE and I-JEPA might overfit on semantic content in arbitrary locations (e.g, the tail location).

In this work, we propose to address location uncertainty in MIM by turning existing MIM models into stochastic ones. Instead of training the model to make predictions in exact locations, we use Stochastic Positional embeddings (StoP) to introduce noise to the masked token's positions, implicitly forcing the model to make stochastic predictions. StoP guides the model towards learning features that are more resilient to location uncertainties, such as the fact that a tail exists in a general area rather than a specific point, which improves downstream performance (Figure~\ref{fig:teaser}b).

Specifically, we model the position of every masked token as a random variable with a Gaussian distribution where its mean is the position of the patch, and the covariance matrix is learned. We find it crucial to design StoP carefully so that the model does not collapse back to deterministic positional embeddings by scaling down the covariance matrix weights to overcome the noise. 

To prevent collapse, we propose to tie between the scales of the noise and input context. With this constraint, scaling down the noise also scales down the input context, which makes the reconstruction task too hard to achieve. On the other hand, increasing the scale of the noise leads to very stochastic masked token positions, which makes the reconstruction task difficult as well. We provide a theoretical proof, showing that our solution indeed prevents collapse.

Our contributions are as follows. First, we propose the idea of Stochastic Positional embeddings (StoP) and apply it to MIM to address the location uncertainty in MIM, namely that the location of semantic features is stochastic. Second, we demonstrate that adding StoP to I-JEPA, a recent MIM approach, leads to improved performance on a variety of downstream tasks, highlighting its effectiveness. Lastly, implementing StoP for MIM requires only three extra lines of code, without adding any runtime or memory overhead.

\section{Preliminaries - Masked Image Modeling}
\label{sec:mim}
The idea in MIM is to train a model to reconstruct masked parts in an image given the rest of the image as context. In this process, a neural network $f_{\theta}$ learns the context representations, and a network $g_{\phi}$ is used to reconstruct the masked regions. In this section we describe the MIM algorithm, then discuss how to apply StoP to MIM in Section~\ref{sec:stop}. 

\textbf{Patchification.} Given an image, the first stage is to tokenize the image. For the case of Vision Transformers~\cite{dosovitskiy2020image}, an input image $I_x\in \mathbb{R}^{H\times W\times 3}$ is first patchified into a sequence of non-overlapping image patches $\hat{p}=(\hat{p}_1,...,\hat{p}_k)$
where $\hat{p}_i \in\mathbb{R}^{H'\times W'\times 3}$ and $K=\frac{HW}{H'W'}$ is the number of patches. Then, each patch $\hat{p}_i$ is projected to $\mathbb{R}^{d_{e}}$ through a linear fully connected layer and its corresponding positional embedding features are added to it, resulting in the patchified set $p=\{p_1, ... p_K\}$.

\textbf{Masking.} Let $x = \{p_i | i \in B_x\}$
be the set of context patches where $B_x$ denotes the set of context indices (i.e.,, the visible tokens in Figure~\ref{fig:architecture}). We denote by $B_y$ the indices of the target patches $y$. The context and target patches are chosen via random masking as in~\citet{he2021masked} or by sampling target continuous blocks as in~\citet{assran2023self}. 

\textbf{Context encoding.} The context tokens are processed via an encoder model $f_{\theta}$ to obtain deep representations: ${s}_x = f_{\theta}(x)$, where $s_{x_i} \in\mathbb{R}^{d_e}$ is the $i^{th}$ context token representation. Each token $s_{x_i}$ is then projected from the output dimension of the encoder $d_e$ to the input dimension of the predictor $d_p$ via a matrix $B\in\mathbb{R}^{d_p \times d_e}$, and it is enriched with deterministic positional embedding $\psi_i \in \mathbb{R}^{d_p}$:
\begin{equation}
\label{eq:mim_context_tokens}
c_i = \psi_i + Bs_{x_i}
\end{equation}
\textbf{Masked tokens.}
We define the set of masked tokens, where every masked token $m_j$ for $j \in B_y$ is composed of the positional embeddings of the $j^{th}$ patch $\psi_j$ and a bias term $\Tilde{m}$ that is shared across all masked tokens, namely:
\begin{equation}
\label{eq:mim_masked_tokens}
m_j = {\psi}_j + \Tilde{m}
\end{equation}
\textbf{Prediction and loss.} Finally, the predictor function $g_{\phi}$ is applied to predict the target features $\hat{s}_y = g_{\phi}(c, m)$. To supervise the prediction, the ground truth $s_y=\{s_{y_i}\}_{i\in B_y}$ is obtained either by using the raw RGB pixels or via a latent representation of the pixels. The loss $\frac{1}{\lvert B_y \rvert}\sum_{i \in B_y}L(s_{y_i}, \hat{s}_{y_i}) $ is then applied to minimize the prediction error.

\begin{figure*}[t!]
\centering
\includegraphics[width=1\linewidth]{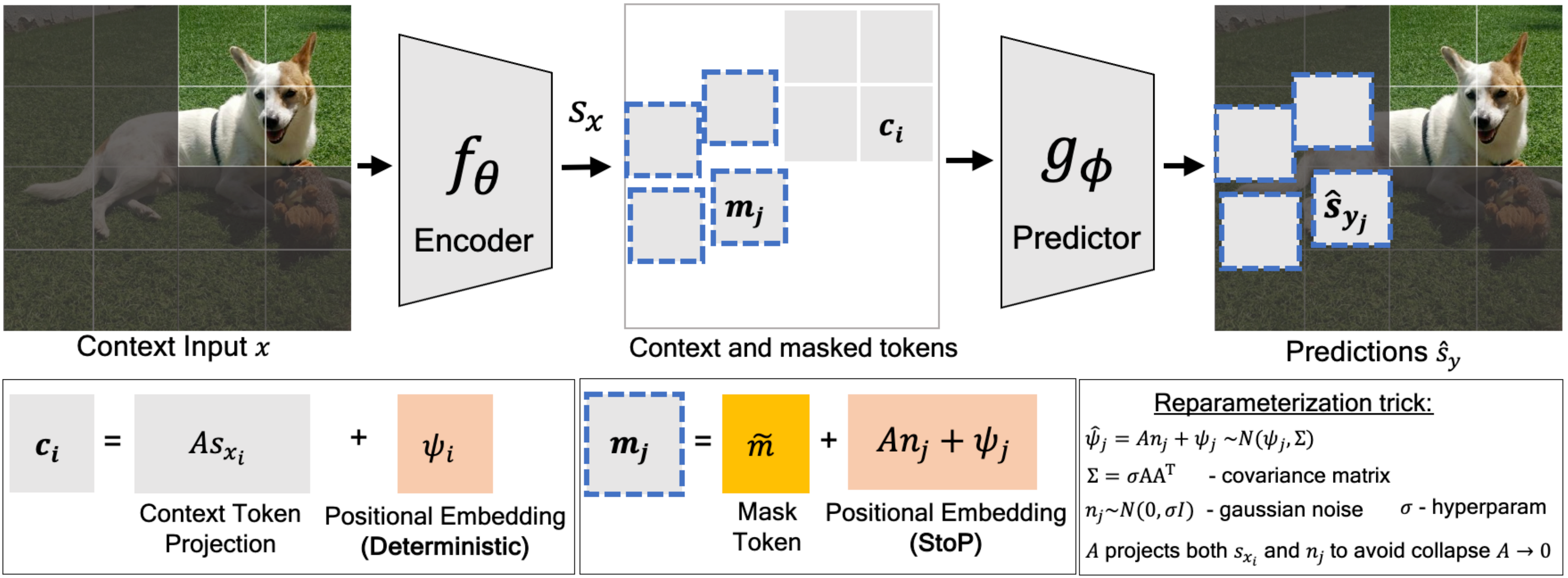}
\caption{\small\textbf{Masked image modeling using stochastic positional embeddings (StoP).} 
$g_{\phi}$ predicts target tokens given masked tokens with stochastic positions $m_j$ and context tokens $c_i$ obtained via $f_{\theta}$. StoP is applied to masked tokens only, leading to features that are more robust to location uncertainties.}
\label{fig:architecture}
\end{figure*}

\section{Masked Image Modeling with~StoP}
\label{sec:stop}
This section presents the StoP formulation, and how to utilize it in MIM while avoiding collapsing back to deterministic positional embeddings. A high-level schematic view of the model is included in Figure~\ref{fig:architecture}, and a pseudo-code implementation is included in Algorithm \ref{alg:fp}.

\textbf{Stochastic Positional Embeddings (StoP)}. Instead of training the model to make predictions in exact locations, we propose to use stochastic positional embeddings which implicitly force the model to make stochastic predictions. This is meant to teach the model that locations cannot be predicted precisely, resulting in improved robustness.

Formulating StoP requires defining the distribution of the stochastic positions, parameterizing it appropriately, and implementing measures to prevent the model from scaling down the noise to the point where it becomes negligible. 

Given a position $j$, we denote by $\hat{\psi}_j$ the random variable providing the position embedding. We assume that $\hat{\psi}_j$ is distributed as Gaussian whose mean is the fixed embedding $\psi_j$, and whose covariance matrix is $\Sigma\in\mathbb{R}^{d_p\times d_p}$:
\begin{equation}
\label{eq:psi}
\hat{\psi}_j \sim N(\psi_j, \Sigma)
\end{equation}
Naturally, we want to learn an optimal $\Sigma$. To parameterize $\Sigma$, we use a general formulation of a low-rank covariance matrix:
\begin{equation}
\Sigma = \sigma AA^{T}\numberthis \label{eqn:sigma}
\end{equation}
Where $A\in\mathbb{R}^{d_{p}\times d_{e}}$ is a learned matrix and $\sigma \in\mathbb{R^+}$ is a positive scalar hyperparameter used to control the Noise to Signal Ratio (NSR).\footnote{At this point, it may seem unnecessary to have an additional $\sigma$ parameter. However, later we will tie $A$ to other model parameters, and thus $\sigma$ will not be redundant and determine the scale of the noise.} By learning the matrix $A$, this formulation allows assigning different noise levels to different location components (e.g., high and low resolution), as well as capturing correlations between location features. 

Using this formulation is challenging for two reasons. First, the sampling process of $\hat{\psi}$ is non-differential w.r.t $A$, and therefore we cannot derive gradients to directly optimize it with SGD. Second, learning might result in the optimization process setting the values of $\Sigma$ to zero, leading to no randomness. Next, we move to solve these issues.

\textbf{Reparametrization Trick.} Since $\hat{\psi}_j$ is sampled from a parameterized distribution, it is non-differentiable in $A$. 
However, a standard trick in these cases is to reparameterize the distribution so that the sampling is from a fixed distribution that does not depend on $A$ (e.g., see \citet{kingma2013auto}). Specifically, we generate samples from $\hat{\psi}_j$ by first sampling a vector $n_j\in\mathbb{R}^{d_e}$ from a standard Gaussian distribution: $n_j \sim N(0, \sigma I)$. Then, $\hat{\psi}_j$ is set to:
\begin{equation}
% \vspace{-1pt}
\label{eqn:reparam}
\hat{\psi}_j = An_j + \psi_j
\vspace{-5pt}
\end{equation}
The resulting distribution of $\hat{\psi}_j$ is equal to that in Equation~\ref{eq:psi}, however, we can now differentiate directly through $A$. 

\textbf{Collapse to deterministic positions (A=0).} Intuitively, adding noise to an objective hurts the training loss, and thus if $A$ appears only in \eqref{eqn:reparam}, training should set it to zero. We indeed observe this empirically, suggesting that $A$ cannot only appear in a single place in the model. In what follows we propose an approach to overcoming this issue.

\begin{algorithm}[t]
  \caption{MIM w/ StoP pseudo-code. requires only a minor implementation change, highlighted in light gray.}
  \label{alg:fp}
  \small
  \begin{algorithmic}[1]
     \State \textbf{Input:} num iterations $K$, image dist $S$, hyperparam $\sigma$, positional embeddings $\psi$
     \State \textbf{Params}: ${A,\Tilde{m}}$, encoder $f_\theta$, predictor $g_\phi$
    \For {$itr=1,2,...,K$}
        \State $I_x \sim S$ 
        \State $p \leftarrow \text{patchify}(I_x)$ 
        \State $(x, B_x),(y, B_y) \leftarrow \text{mask}(p)$ 
        \State $s_x \leftarrow f_{\theta}(x)$
        \State \codecomment{\# apply StoP on a sequence of tokens}
        \State \texttt{\colorbox{lightgrey}{$n_j \sim \mathcal{N}(0, \sigma I$)}}
        \State \codecomment{\# $\psi_{B_x}$, $\psi_{B_y}$ - masked/context positional embeddings}
        \State $m = $ \texttt{\colorbox{lightgrey}{$An$}} $+ \psi_{B_y} + \Tilde{m}$ \State $c = As_x + \psi_{B_x}$
        \State \codecomment{\# predict targets}
        \State $\hat{s}_y \leftarrow g_\phi(c, m)$ 
        \State $ s_y \leftarrow \text{get\_target}(y)$
        \State $\text{loss} \leftarrow L(\hat{s}_y, s_y)$
        \State $\text{sgd\_step}(\text{loss}; \{\theta,\phi, A, \Tilde{m} \})$ 
    \EndFor
  \end{algorithmic}
\end{algorithm}

\textbf{Avoiding collapse by weight tying A=B.} To avoid the collapse to deterministic positions, we propose to tie the weights of $A$ and $B$ (originally defined in Eq.~\ref{eq:mim_context_tokens}), such that the same matrix $A$ projects both the context tokens $s_{x_i}$ and the noise tokens $n_j$: 
\begin{equation}
\label{eq:tie}
c_i = As_{x_i} + \psi_i \quad m_j = An_j + \psi_j + \Tilde{m} 
\end{equation}
This tying means that the scale of the noise and the input are both determined by $A$, and thus the noise cannot be set to zero, without affecting other parts of the model. This can be understood by considering two extreme cases:
\begin{itemize}[itemsep=0.5pt]
\vspace{-2mm}
    \item If $A=0$, there is complete certainty about the positional embeddings but all context is lost ($As_{x_i}=0$).
    \item If $A$ has large magnitude, the context information is preserved but the noise is amplified and camouflages masked tokens positional embeddings ($An_j \gg \psi_j$).
\end{itemize}
 This dual role of $A$ forces the model to trade-off between the positions of the masked tokens and the context tokens.\footnote{Note that an implicit assumption here is that $\psi$ and $s_x$ have fixed magnitude. This is true for sine-cosine features and for $s_x$ which are layer normalized by the transformer last layer.}

In the following proposition, we formally show that if the weights $A$ and $B$ are tied then $A$ cannot collapse. More specifically, $A=0$ occurs only if in the original deterministic setting $B$ goes to zero and doesn't utilize the context anyway. Formally, consider a regression task where $F$ predicts some target $y_j$ given a stochastic position $An_j + \psi_j + \Tilde{m}$ where $n_j \sim N(0, \sigma I)$ and projected context token $Bx_i$. Denote $J_{tied}, J_{det}$ the loss functions when tying the weights $A$ and $B$, and when using deterministic positional embeddings respectively:
\begin{equation*}
J_{tied}(A) = \sum_{i,j} \mathbb{E}_{n_j}[(F(An_j + \psi_j + \Tilde{m}, Ax_i) - y_j)^2]
\end{equation*}
\begin{equation*}
J_{det}(B) = \sum_{i,j} [(F(\psi_j + \Tilde{m}, Bx_i) - y_j)^2]
\end{equation*}

\begin{proposition}
If the weights of $A$ and $B$ are tied (namely $A=B$) then $\left. \frac{dJ_{tied}}{dA} \right|_{A=0} = 0$ iff $\left. \frac{dJ_{det}}{dB} \right|_{B=0} = 0$
\end{proposition} 
Proof is included in Appendix~\ref{sup:proof_collapse}. 

\textbf{Optimal Predictor}. Our approach relies on using stochastic positional embeddings. Here we provide further analysis, showing that the optimal predictor performs spatial smoothing. Consider a random variable $X$ (corresponding to the context in our case. For simplicity assume $X$ is just the positional embedding of the context) that is used to predict a variable $Y$ (corresponding to the target in our case). But now instead of predicting from $X$, we use a noise variable $Z$
that is independent of both $X,Y$, and provide the predictor with only the noisy result $R = g(X,Z)$. Here $g$ is some mixing function (in our case $g(x,z) = x+z$). We next derive the optimal predictor $f(R)$ in this case. Formally we want to minimize:
\begin{equation}
\label{eq:optimal_objective}
    E_{R,Y}[(f(R) - Y)^2]
\end{equation}
\begin{proposition}
If $Z$ is a Gaussian with zero mean and unit variance, the optimal predictor that minimizes Equation~\ref{eq:optimal_objective} is: 
$$f(r) = \int_x E[Y|X=x]\frac{1}{\sqrt{2\pi}}e^{-0.5(x-r)^2}dx $$
\end{proposition}
Thus, the optimal predictor amounts to a convolution of the clean expected values with a Gaussian. See Appendix~\ref{sup:optimal} for the proof.

\section{Experiments and Results}
Next, we turn to discuss the main experiments presented in the paper. In Section~\ref{subsec:exp:downstream}, we describe the application of StoP to various downstream tasks including image recognition, dense prediction, and low-level vision tasks. In Section~\ref{subsec:exp:ablation} we  discuss the ablation study and design choices. The full implementation details are included in Appendix~\ref{app:experiments}.

\subsection{Downstream Tasks}
\label{subsec:exp:downstream}
We conducted pre-training of StoP on top of I-JEPA, which is a state-of-the-art MIM model. We train on IN-1k for a period of $600$ epochs using  ViT-B/16 and ViT-L/16 architectures for the encoder and predictor or for $300$ epochs when using ViT-H/14. Subsequently, we proceeded to evaluate the model's performance on a variety of downstream tasks. Additional results and comparison to invariance-based approaches are included Appendix~\ref{app:downstream}.
\begin{table*}
\centering
\begin{tabular}{llccc}
        \bf\small Arch & \bf\small Method & \bf\small 1\%, last layer & \bf\small 100\%, last layer & \bf\small 100\%, last 4 layers\\
        \toprule
        \multirow{2}{*}{\small ViT-B/16} & \small I-JEPA & \small   57.1 & 70.9  & 72.9 \\
        & \small +StoP & \cc\small 60.3 (+3.2\%) & \cc 72.6 (+1.7\%) & \cc  74.5 (+1.6\%)\\
        \midrule
        \multirow{2}{*}{\small ViT-L/16} &\small I-JEPA & \small  64.2 & 76.1 & 77.5 \\
        & \small +StoP & \cc\small 65.1 (+0.9\%) & \cc 77.1 (+1.0\%) & \cc 78.5 (+1.0\%)\\
        \midrule
        \multirow{2}{*}{\small ViT-H/14} & \small I-JEPA & \small 62.9  & 78.2 & 79.3 \\
        & \small +StoP & \cc\small 65.4 (+2.5\%)  & \cc79.0 (+0.8\%) & \cc79.6 (+0.3\%) \\
        \bottomrule
  \end{tabular}
    \caption{\small{\bf StoP compared to deterministic sinusoidal positional embeddings on IN-1k}. StoP leads to consistent linear probing improvement in all settings. When applying linear probing on a trained ViT-H model with StoP, using only $1\%$ of the labeled data and using averaged pooled features from the last layer, StoP results in an +2.5\% improvement. The baseline I-JEPA uses sinusoidal positional embeddings.}
    \label{tb:strict_lin}
\end{table*}

\textbf{Image recognition.} For image classification, we perform a linear probing evaluation of StoP on multiple datasets, including ImageNet (IN-1k)~\citep{russakovsky2015imagenet}, Places 205~\citep{places205}, iNaturalist 2018~\citep{van2018inaturalist}, and CIFAR 100~\citep{cifar10}. These datasets vary in their size, their purpose, and the geographical environments from which the images were captured. For example, IN-1k contains over $1.2$ million images compared to CIFAR-100 which contains only $60,000$ images, and while IN-1k is focused on object recognition, iNaturalist and Places are focused on scene and species recognition.

\begin{table}[t]
\centering
\small
\begin{tabular}{lllc}
        \bf\small Method & \bf\small Arch. & \bf\small Epochs & \bf\small Top-1\\
        \toprule
        data2vec & \small ViT-L/16 & 1600 & 77.3 \\[1ex]
        \multirow{2}{*}{\small MAE} & \small ViT-B/16 & 1600 & 68.0\\
        & \small ViT-L/16 & 1600 & 75.8\\
        & \small ViT-H/14 & 1600 & 76.6\\
        [1ex]
        \multirow{2}{*}{\small I-JEPA} 
        & \small ViT-B/16 & 600 & {70.9}\\
        & \small ViT-L/16 & 600  & {76.1}\\ 
        & \small ViT-H/14 &  300 & {78.2}\\[1ex]
        \multirow{2}{*}{+StoP (ours)} 
        & \cc\small ViT-B/16 & \cc 600 & \cc {72.6}\\
        & \cc\small ViT-L/16 & \cc 600  & \cc{77.1}\\
        & \cc\small ViT-H/14 &  \cc 300 & \cc\bf{79.0}\\
        \bottomrule
    \end{tabular}
    \caption{\small{\bf Linear-evaluation on IN-1k}. Replacing sinusoidal positional embeddings with StoP in I-JEPA significantly improves linear probing results.}
    \label{tb:lineareval}
    \vspace{-2mm}
\end{table}
\begin{table}[t]
    \centering
    \small
    \begin{tabular}{llccc}
        \bf\small Method & \bf\small Arch. & \bf\small J-Mean & \bf\small F-Mean & \bf\small J\&F Mean\\
        \toprule
        \multirow{3}{*}{\small MAE} & \small ViT-B/16 & 49.4 & 52.6 & 50.9 \\
        & \small ViT-L/16 & 52.5 & 54.3 & 53.4\\
        & \small ViT-H/14 & 54.0 & 57.0 & 55.5 \\ [1ex]
        \multirow{2}{*}{\small I-JEPA} & \small ViT-B/16 &  56.1 & 56.2 & 56.1 \\
        & \small ViT-L/16 & 56.1 & 55.7 & 55.9 \\ 
        & \small ViT-H/14 & 58.5 & 60.9 & 59.7 \\ [1ex]
        \multirow{3}{*}{\small +StoP} & \cc\small ViT-B/16 & \cc 56.6 & \cc 57.3 & \cc 57.0 \\
        & \cc\small ViT-L/16 & \cc 58.1 & \cc 58.7 & \cc 58.4\\
        & \cc\small ViT-H/14 & \cc \bf 58.9 & \cc \bf 61.2 & \cc \bf 60.1\\
        \bottomrule
\end{tabular}
% }
    \caption{\small{\textbf{Video objects semi-supervised segmentation.} MIM with StoP learns features with a finer level of granularity. Results are reported on DAVIS 2017 dataset. }}
  \label{tb:labelprop}
\end{table}

In Table~\ref{tb:strict_lin}, we present the linear probing image classification results conducted on IN-1k under different linear evaluation protocols using different amounts of data, and by aggregating features from different layers. E.g, ``100\%, last 4 layers'' applies linear probing on the entire IN-1k data and the representation of each image is comprised of a concatenation of four feature vectors, each one summarizes information from its corresponding layer via average pooling. In Table~\ref{tb:lineareval} we compare linear probing results of common MIM methods on IN-1k, reporting past published performance. In Table~\ref{tb:lineareval} all perform linear probing over the output from the last layer.

StoP improves the baseline performance using all architectures examined. For example, $+2.5\%$ linear probing performance gains with ViT-H using $1\%$ of the labeled data and $1.6\%$ when using features from the last $4$ layers using ViT-B on the full IN-1k data. Furthermore, using~StoP leads to improvements in downstream linear probing tasks (see Table~\ref{tb:transfer-classification}). For example, StoP leads to $3.3\%$ improvement on iNAT using ViT-H and 1.3\% on counting. This confirms that the learned representations lead to improvements in a large variety of image recognition tasks. On full finetuning using 1\% of the labeled data, we observe similar performance improvements (see Table~\ref{tab:finetune}), e.g, $+2.3\%$ improvements on Top-1 accuracy using ViT-L model. We provide the full finetuning results in Table~\ref{supp:tab:finetune}, Appendix~\ref{app:downstream}.

\begin{table*}
    \centering
    \small
    \begin{tabular}{llcccccc}
            \bf\small Method & \bf\small Arch. & \bf\small CIFAR100 & \bf\small Places205 &  \bf\small iNat18 & \bf\small CLEVR/Count & \bf\small CLEVR/Dist  \\
        \toprule
        \small data2vec & \small ViT-L/16 & 81.6 & 54.6 & 28.1 & 85.3 & 71.3 \\[1ex]
        \multirow{3}{*}{\small MAE} & \small ViT-B/16 & 68.1 & 49.2 & 26.8 & 86.6 & 70.8 \\
         & \small ViT-L/16 & 77.4 & 54.4 & 33.0 & \textbf{92.1} & 73.0 \\
          & \small ViT-H/14 & 77.3 & 55.0 & 32.9 & 90.5 & 72.4 \\
         [1ex]
        \multirow{3}{*}{\small I-JEPA} & \small ViT-B/16 & 69.2 & 53.4 & 43.4 & 82.2 & 70.7 \\
        & \small ViT-L/16 & 83.6 & 56.5 & 48.4 & 85.6 & 71.2\\
        & \small ViT-H/14 & 87.5 & 58.4 & 47.6 & 86.7 & 72.4 \\[1ex]
        \multirow{3}{*}{\small +StoP} & \small \cc ViT-B/16 & \cc 81.2 & \cc 54.3 & \cc 44.7 & \cc 83.7 & \cc 71.3\\
        & \small \cc ViT-L/16 & \cc 84.7 & \cc 57.2 & \cc 49.2 & \cc 85.7 & \cc 70.2 \\
        & \small \cc ViT-H/14 & \cc \bf 87.7 & \cc \bf \textbf{58.4} & \cc \bf 50.9 & \cc 88.0 & \cc \bf 72.5 \\
        \bottomrule
    \end{tabular}
    \caption{\small
    \textbf{Linear-probe transfer for various downstream tasks}. Linear-evaluation on downstream image classification, object counting, and depth ordering tasks. Using StoP instead of sinusoidal deterministic positions leads to improvements on all tasks. E.g, $+3.3\%$ on iNAT18 and $+1.3\%$ on Counting.}
    \label{tb:transfer-classification}
\end{table*}

\textbf{Counting and depth ordering.} We assess the downstream performance on tasks that require fine-grained objects representations like counting and depth ordering using the CLEVR~\cite{johnson2017clevr} dataset. Table~\ref{tb:transfer-classification} provides evidence that using StoP significantly improve counting ($+1.3\%$) and slightly improve depth ordering ($+0.1\%$).

\textbf{Dense prediction.} To evaluate how well StoP performs on dense prediction tasks, e.g, tasks that require fine-grained spatial representations, we utilized the learned models for semi-supervised video object segmentation on the DAVIS 2017~\citep{pont20172017} dataset. We follow previous works (e.g~\citet{jabri2020space, caron2021emerging}) and use the pretrained model to extract frames features and use patch-level affinities between frames to track the first segmentation mask. We include video semi-supervised video-object segmentation by tracking results in Table~\ref{tb:labelprop}.  We find that StoP significantly improves over I-JEPA with deterministic sinusoidal location features. For example, we observe an improvement of $+2.5\%$ in $J\&F$ using ViT-L.

\subsection{Ablation Study} 
\label{subsec:exp:ablation}
\begin{table}[t]
  \centering
  \begin{tabular}{llc}
        \bf\small Method & \bf\small Epochs & \bf\small Top-1\\
        \toprule
        \multirow{1}{*}{\small Sine Cosine} 
        & 600  & {69.4}\\ 
        \multirow{1}{*}{StoP (ours)} 
        & \cc 600  & \cc{\textbf{71.7}}\\
        \bottomrule
    \end{tabular}
  \caption{\small{\textbf{Finetuning results over IN-1k with 1\% labels.} StoP significantly improves finetuning performance compared to using sine-cosine positional embeddings. Using ViT-L/16 architecture.}}
    \label{tab:finetune}
\end{table}

\begin{figure}[t]
    \centering
    \includegraphics[width=1\linewidth]{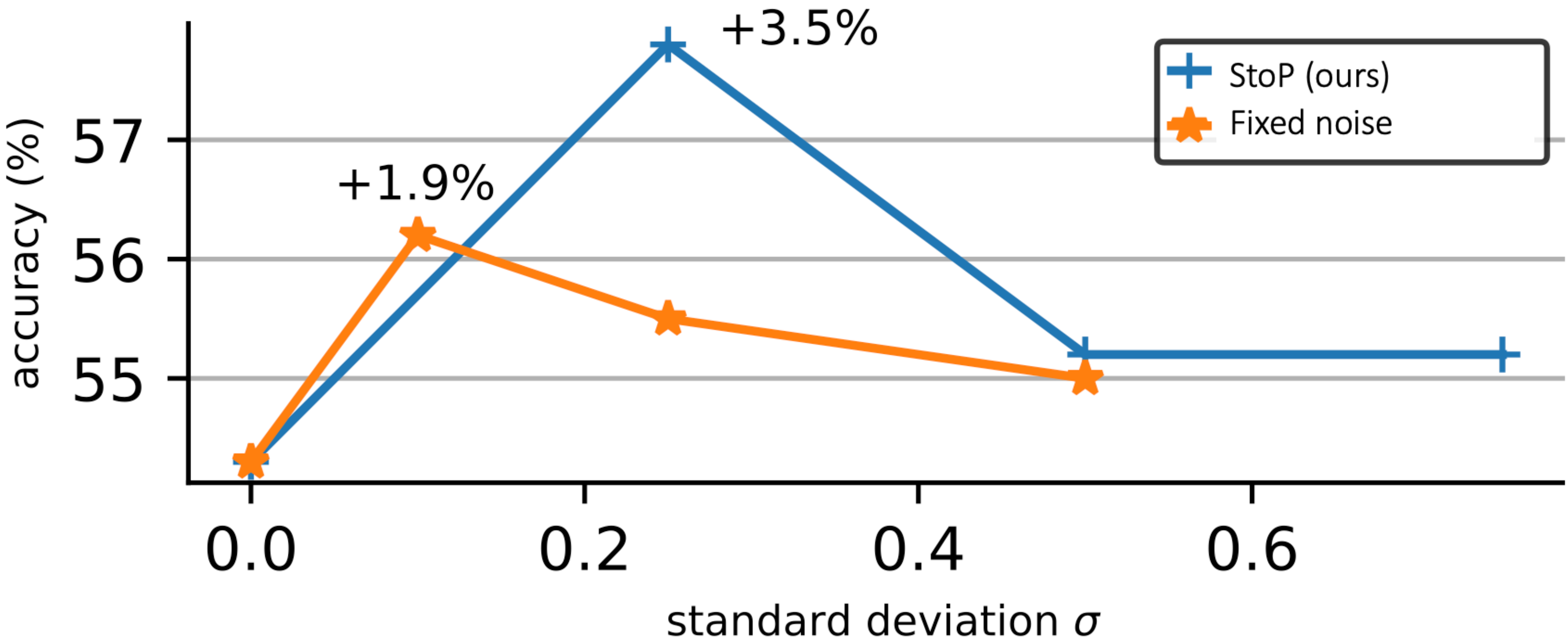}
    \caption{\small\textbf{Learned vs. predefined stochastic positions.} Using the learned covariance matrix as in StoP, e.g, $\Sigma=\sigma AA^T$ leads to $+3.5\%$ improvement compared to smaller gains with a fixed covariance matrix $\Sigma=\sigma I$. Accuracy is reported based on linear probing evaluation using 1\% of the data from IN-1k.
}
\label{fig:noise}
\end{figure}
\begin{table}[t]
\vspace{-4mm}
    \centering
        \begin{tabular}{l c}
        \bf\small Method & \bf\small Top-1 \\
        \toprule
        \small Sine Cosine & 54.3\\    
        \small Learned Pos. Embedding & 54.4 \\
        \small Stochastic Positions (StoP) &   \cc\bf 57.8 \\
    \end{tabular}
    \captionof{table}{
\small\textbf{Different positional embeddings}.
Linear probing on IN-1K using only 1\% of the labels. Stochastic Positions (StoP) outperforms other common deterministic variants by $3.3\%$.
}    
  \label{tb:diff_pos}
\end{table}

Our primary focus is to evaluate the effectiveness of StoP. To demonstrate this, we assess various design options using ViT-B architecture for the encoder and predictor. We pre-train for $300$ epochs on IN-1k based on the I-JEPA~\citep{assran2023self} MIM model. We then assessed the linear probing performance on IN-1k using only 1\% of the labels. 

\textbf{StoP compared to deterministic positional embeddings.} The most common choices for positional embeddings for Vision Transformers are sine-cosine location features (also used in MAE, I-JEPA) and learned positional embedding. We evaluate the MIM downstream performance using each of these options and using~StoP (see Table~\ref{tb:diff_pos}). The results indicate that using StoP improves the performance by $+3.2\%$ compared to sinusoidal and learned positional embeddings.

\textbf{Learned vs. predefined covariance matrix.} To confirm that learning the covariance matrix $\Sigma = \sigma AA^T$ (and specifically $A$) is beneficial compared to using a predefined covariance matrix, we compare to stochastic positional embeddings with a predefined covariance matrix $\Sigma=\sigma I$, without any learning. We compare both options using different $\sigma$ hyperparameter values. Figure~\ref{fig:noise} indicates that it is advantageous to learn $\Sigma$ rather than use fixed parameters. Our findings show that setting the hyperparameter value to $\sigma=0.25$ leads to an improvement of $3.5\%$ points compared to deterministic positional embeddings ($\sigma=0$).

\textbf{Application of StoP to different tokens.} 
We apply StoP to context and/or masked tokens. The results in Table~\ref{tb:ablation-context} confirm our design choice, showing that StoP is most beneficial when it is applied solely to masked tokens, compared to context tokens, or both masked and context tokens.

\begin{figure}[t]
\centering
\includegraphics[width=0.85\linewidth]{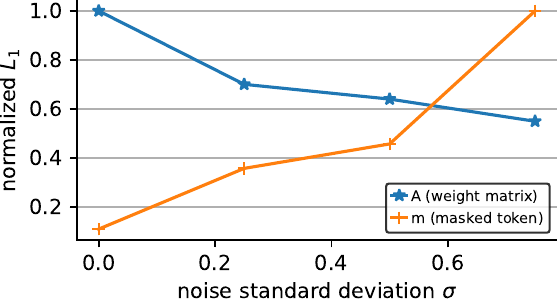}
\caption{\small\textbf{Increasing $\sigma$ induces regularization.} Changing the prior $\sigma$ (where $\Sigma=\sigma AA^T$) induces regularization over $A$ and increases the norm of the masked token, which preserves the masked token information in comparison to the added noise.} 
\label{fig:reg}
\end{figure}

\begin{table}[t]
    \centering
    \begin{tabular}{l c}
        \bf\small Method & \bf\small Top-1 \\
        \toprule
        \small No Noise (Sine Cosine) & 54.3\\
        \small Context tokens only &   55.1 \\
        \small Masked + context tokens & 56.8\\
        \small Masked tokens only & \cc \bf 57.8 \\
    \end{tabular}
    \captionof{table}{\small\textbf{Applying noise to different tokens}. Applying learned noise to context and/or masked tokens positional embeddings (sine-cosine). Reporting linear evaluation accuracy (using 1\% of IN-1k).}
    \label{tb:ablation-context}
\end{table}

\subsection{Analysis}
To explain how StoP affects MIM, we analyze the learned model weights, visualize the stochastic positional embeddings, and visualize the predicted features.

\textbf{StoP induces regularization.} The matrix $A$ is used to project both noise tokens and context embedding tokens. We hypothesize that StoP implicitly regularizes $A$. To test this hypothesis we train models using StoP changing only the hyperparam $\sigma$ (see Figure~\ref{fig:reg}). We find that increasing the value of $\sigma$ leads to a decrease in the norm of $A$, which can be viewed as regularization. On the other hand, increasing $\sigma$ leads to an increase in the norm of the masked token bias $\Tilde{m}$. We speculate that the masked token bias increases in scale to prevent losing its information relative to the noise.

To further analyze this phenomenon, we train additional models while applying $l_1$ or $l_2$ regularization on $A$ while keeping the positional embeddings of masked tokens deterministic. We find that StoP leads to +$2\%$ improvement over $l_1$ and +$2.1\%$ over $l_2$ regualrization. Therefore, we conclude that StoP is superior to simple regularization.

\textbf{Stochastic positional embedding visualization.} 
\begin{table}
  \centering
    \centering
    % \resizebox{0.35\textwidth}{!}{
    \begin{tabular}{l c}
        \bf\small Method & \bf\small Top-1 \\
        \toprule
        \small Sine Cosine & 54.3\\    
        \small x2 Low res (bilinear resize) & 52.1 \\
        \small x2 Low res (max pooling) & 54.1 \\
        \small Stochastic Positions (StoP) &   \cc\bf 57.8 \\
    \end{tabular}
    % }
    \captionof{table}{\small\textbf{Low resolution prediction}.
Performance of StoP compared to models that predict features on lower scales via max pooling or bilinear resizing. Reporting linear evaluation accuracy (using 1\% of IN-1k). StoP performs better than low res prediction.}
\label{tb:low_res}
\vspace{-2mm}
\end{table}
To visualize how StoP affects the similarity between different positions, we plot the similarity matrix between a stochastic position embedding query and the predefined sine-cosine deterministic positions (Figure~\ref{fig:similarity}). With StoP, we find that query locations are more similar to a wider range of neighboring locations. Building on this observation, we train models to investigate if directly predicting lower-scale features is beneficial. We trained models to predict features in both the original scale and a downscaled version by a factor of 2, using bilinear resizing and max pooling for downscaling. However, we found that predicting lower scale features does not improve performance (see Table~\ref{tb:low_res}).
\begin{figure}[t]
\centering
\includegraphics[width=0.95\columnwidth]{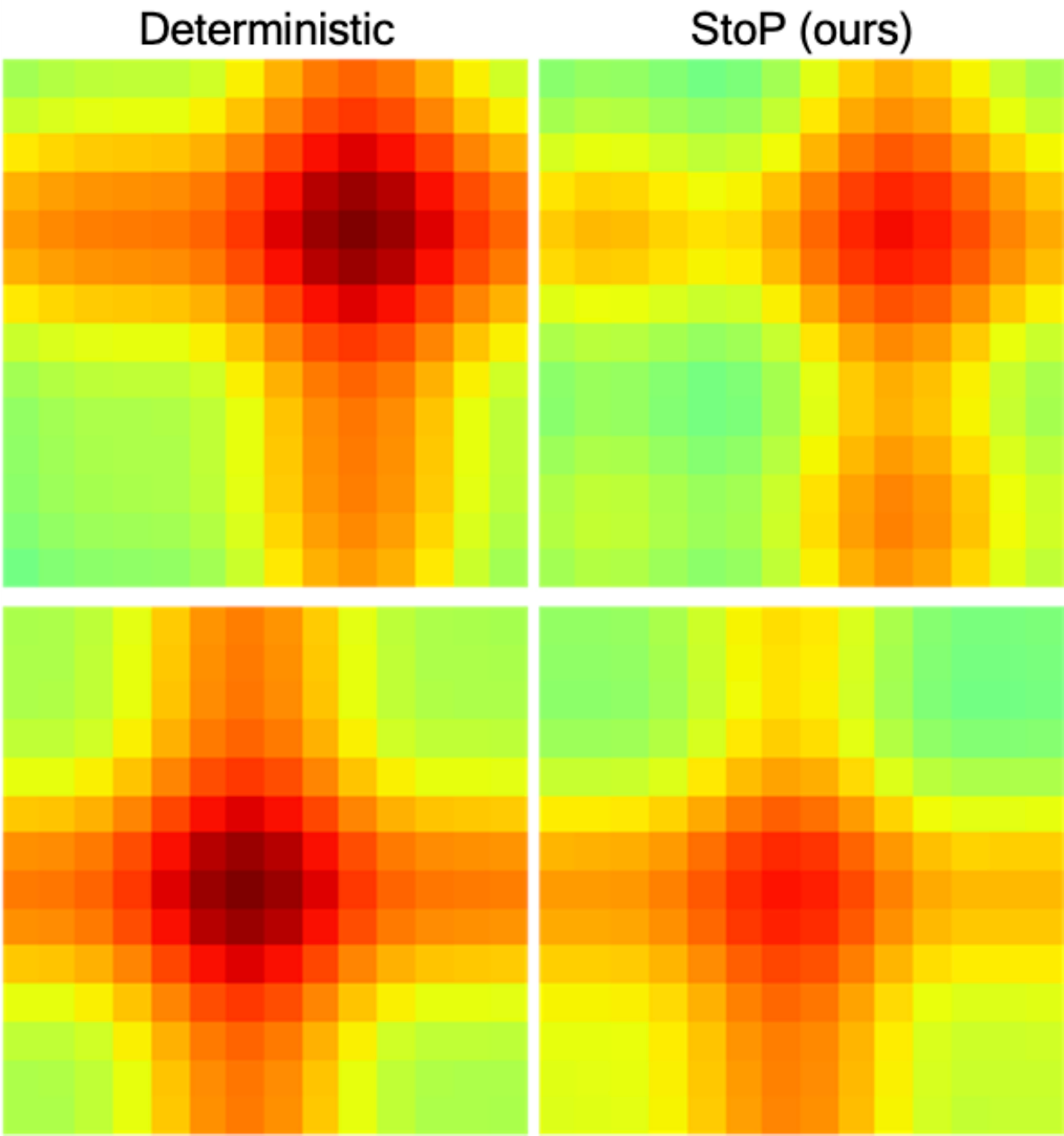}
\caption{\small \textbf{Similarity matrices of deterministic and stochastic positional embedding (StoP) to a query position}. Each row represents the similarity given a different query position. StoP leads to a spatially smooth similarity matrix, thereby making it hard to distinguish the exact location of a given patch.
}
\label{fig:similarity}
\end{figure}
\begin{figure}[t]
\centering
\includegraphics[width=0.95\linewidth]{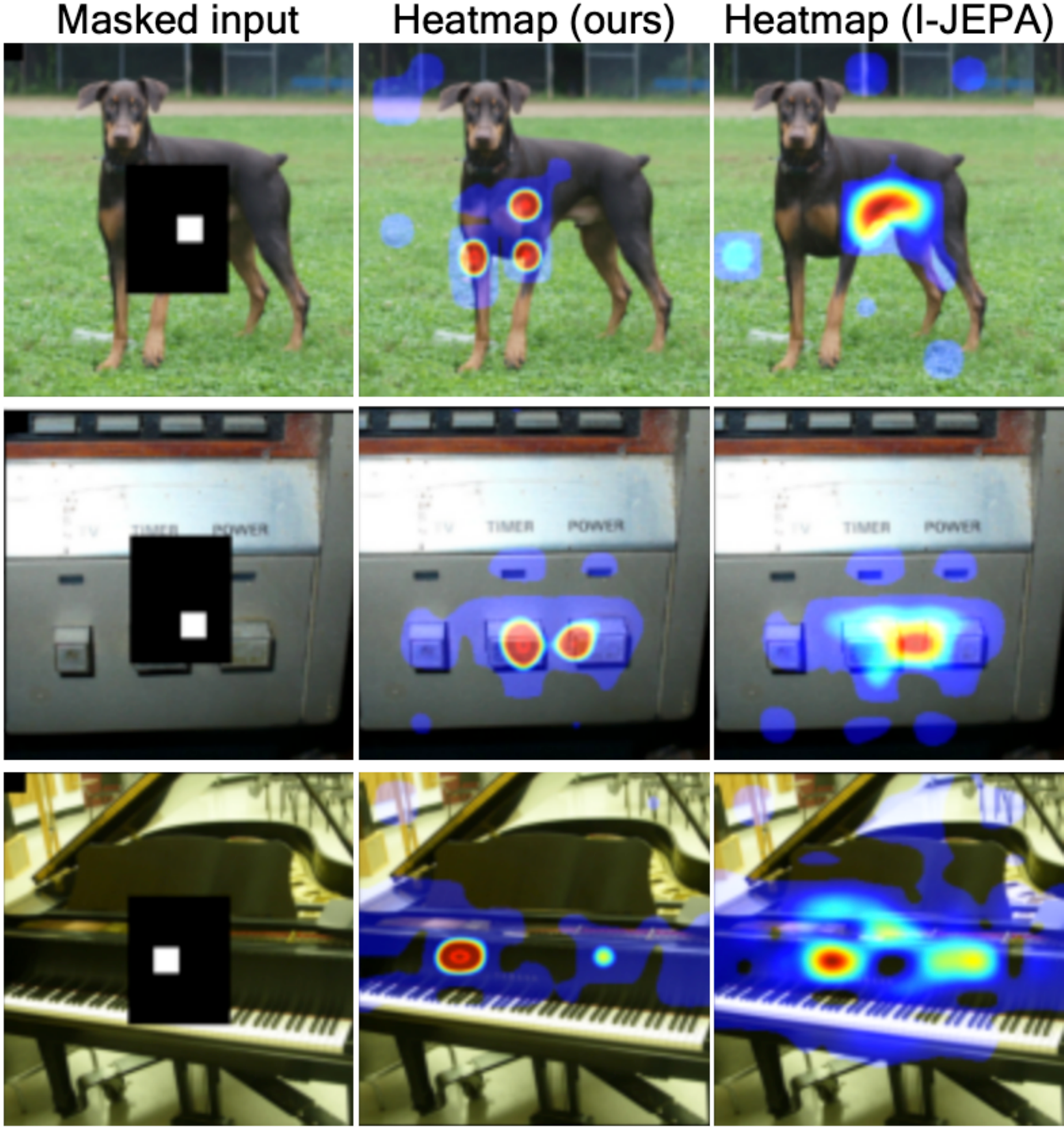}
\captionof{figure}{\small\textbf{Feature visualization}. We plot the similarity between the predicted features of a given patch (marked in white within the masked black area) and other features in the same image. Using StoP produces features that are less location based compared to I-JEPA baseline that have strong correlation with the target location.}
\label{fig:viz_encoder}
\vspace{-3mm}
\end{figure}

\textbf{Prediction visualization.} We include heatmap visualization to visualize the similarity of a predicted token to all other tokens within the same image (see Figure~\ref{fig:viz_encoder}). For a given image, mask, and a masked patch of interest, we apply cosine similarity between the predicted patch and all other token representations within the same image, followed by a softmax. For I-JEPA with sine-cosine positional embeddings, the visualization indicates that adjacent tokens tend to share similar features, implying a correlation between the features and spatial location. In contrast, StoP produces predictions correlated with non-neighboring small areas. We speculate that using StoP leads to learning features that are more semantic and prevents overfitting to location features.
\label{sec:analysis}
\section{Related Work}
\textbf{Masked image modeling (MIM).} There is a significant body of research exploring visual representation learning by predicting corrupted sensory inputs. Denoising autoencoders~\citep{vincent2010stacked}, for example, use random noise as input corruption, while context encoders~\citep{pathak2016context} regress an entire image region based on its surrounding. The idea behind masked image modeling~\citep{he2021masked, xie2021simmim, bao2021beit} has emerged as a way to address image denoising. In this approach, a Vision Transformer~\citep{dosovitskiy2020image} is used to reconstruct missing input patches. The Masked Autoencoders (MAE) architecture~\citep{he2021masked}, for example, efficiently reconstructs missing patches in pixel space and achieves strong performance on large labeled datasets. Other approaches, such as BEiT~\citep{bao2021beit}, predict a latent code obtained using a pretrained tokenizer. However, pixel-level pre-training has been shown to outperform BEiT in fine-tuning. SimMiM~\citep{xie2021simmim} explores simple reconstruction targets like color clusters but shows no significant advantages over pixel space reconstruction. Recently, Image-JEPA (I-JEPA)~\citep{assran2023self,lecun2022path} was proposed as a non-generative approach for self-supervised learning of semantic image representations. I-JEPA predicts the representations of various target blocks in an image from a single context block to guide it toward producing semantic representations. Our approach builds on this line of work and we propose to deal with location uncertainty using stochastic positional embeddings which was not explored before.

\textbf{Positional Embeddings in Transformers.} One of the core components of the Transformer architecture~\citep{vaswani2017attention} is the Self-Attention block, which is a permutation invariant function, e.g, changing the order of the input tokens does not change the function output. Consequently, it is necessary to feed input tokens together with their positional embedding to describe their location. Absolute positional embeddings like fixed 2D sinusoidal features~\citep{bello2019attention} or learned location features are the prevalent type of positional embeddings for the Vision Transformer~\citep{dosovitskiy2020image}. Relative positional embeddings have recently gained popularity in NLP due to their ability to address the gap between the training and testing sequence length~\citep{su2021roformer,chu2021conditional,press2021train}. For example, ~\cite{press2021train} proposed ALiBi to bias self-attention to assign higher confidence to neighboring locations, and SPE~\citep{pmlr-v139-liutkus21a} proposed a stochastic approximation for relative positional embedding in linear transformers. Differently, we propose StoP to tackle location uncertainties in MIM, and it can be easily applied on top of any existing deterministic variant.

\textbf{Invariance-based methods.} These methods incorporate a loss that encourages similarity between augmented views of the the same image while avoiding a trivial solution. For example, contrastive learning prevents collapse by introducing negative examples~\citep{Hadsell2006DimensionalityRB, Dosovitskiy2014DiscriminativeUF, chen2020simple, he2019moco, chen2020mocov2, dwibedi2021little}. This can be achieved using a memory bank of previous instances \citep{wu2018unsupervised, oord2018representation, Tian2019ContrastiveMC, misra2020self}. However, there are also non-contrastive solutions that have been proposed. Of particular interest, a momentum encoder has been shown to prevent collapse even without negative pairs~\citep{grill2020bootstrap, caron2021emerging, salakhutdinov2007learning}. Other methods include stopping the gradient to one branch~\citep{chen2021exploring} or applying regularization using batch statistics~\citep{zbontar2021barlow, bardes2021vicreg, bardes2022vicregl, Ermolov2020WhiteningFS, Hua_2021_ICCV}. MoCo~v3~\cite{chen2021empirical}, then DINO~\citep{caron2021emerging} extended these approaches for Vision Transformer, and iBOT~\citep{zhou2021ibotyes} proposed to add a MIM loss to DINO. These approaches perform extremely well on ImageNet linear-probing, yet they rely on batch statistics, struggle under non-uniform distributions~\citep{assran2022hidden}, and require hand-crafted image augmentations~\citep{xiaoshould}. Our approach is based on MIM that requires less assumptions on batch statistics or handcrafted invariances.\label{sec:ijepa}

\section{Limitations}
We applied StoP to I-JEPA which performs image reconstruction in the feature space. However, our attempts to apply StoP to MIM that use pixel based reconstruction, mainly MAE, were not successful. We speculate that adding StoP to MAE might make pixel reconstruction too difficult to achieve. Additionally, StoP tackles location uncertainty but not appearance uncertainty, which we believe is implicitly modeled by reconstructing tokens in feature space. Also, when modeling stochastic positions it may might be possible to condition the noise on the input image, namely the context tokens. We leave this extension for future work. Lastly, while combining StoP with MIM shows significant improvements, invariance-based approaches still perform slightly better (e.g, iBOT, DINO) than MIM approaches.

\section{Conclusion}
In this work, we proposed to use stochastic positional embedding (StoP) to tackle location uncertainty in MIM. By conditioning on stochastic masked token positions, our model learns features that are more robust to location uncertainty. The effectiveness of this approach is demonstrated on various datasets and downstream tasks, outperforming existing MIM methods and highlighting its potential for self-supervised learning. Based on our experiments and visualizations, modeling location uncertainties with StoP reduces overfitting to location features. 

\bibliography{example_paper}

\begin{thebibliography}{49}
\providecommand{\natexlab}[1]{#1}
\providecommand{\url}[1]{\texttt{#1}}
\expandafter\ifx\csname urlstyle\endcsname\relax
  \providecommand{\doi}[1]{doi: #1}\else
  \providecommand{\doi}{doi: \begingroup \urlstyle{rm}\Url}\fi

\bibitem[Assran et~al.(2022)Assran, Balestriero, Duval, Bordes, Misra, Bojanowski, Vincent, Rabbat, and Ballas]{assran2022hidden}
Assran, M., Balestriero, R., Duval, Q., Bordes, F., Misra, I., Bojanowski, P., Vincent, P., Rabbat, M., and Ballas, N.
\newblock The hidden uniform cluster prior in self-supervised learning.
\newblock \emph{arXiv preprint arXiv:2210.07277}, 2022.

\bibitem[Assran et~al.(2023)Assran, Duval, Misra, Bojanowski, Vincent, Rabbat, LeCun, and Ballas]{assran2023self}
Assran, M., Duval, Q., Misra, I., Bojanowski, P., Vincent, P., Rabbat, M., LeCun, Y., and Ballas, N.
\newblock Self-supervised learning from images with a joint-embedding predictive architecture.
\newblock \emph{arXiv preprint arXiv:2301.08243}, 2023.

\bibitem[Bao et~al.(2021)Bao, Dong, and Wei]{bao2021beit}
Bao, H., Dong, L., and Wei, F.
\newblock Beit: Bert pre-training of image transformers.
\newblock \emph{arXiv preprint arXiv:2106.08254}, 2021.

\bibitem[Bardes et~al.(2021)Bardes, Ponce, and LeCun]{bardes2021vicreg}
Bardes, A., Ponce, J., and LeCun, Y.
\newblock Vicreg: Variance-invariance-covariance regularization for self-supervised learning.
\newblock \emph{arXiv preprint arXiv:2105.04906}, 2021.

\bibitem[Bardes et~al.(2022)Bardes, Ponce, and LeCun]{bardes2022vicregl}
Bardes, A., Ponce, J., and LeCun, Y.
\newblock Vicregl: Self-supervised learning of local visual features.
\newblock \emph{arXiv preprint arXiv:2210.01571}, 2022.

\bibitem[Bello et~al.(2019)Bello, Zoph, Vaswani, Shlens, and Le]{bello2019attention}
Bello, I., Zoph, B., Vaswani, A., Shlens, J., and Le, Q.~V.
\newblock Attention augmented convolutional networks.
\newblock In \emph{Proceedings of the IEEE/CVF international conference on computer vision}, pp.\  3286--3295, 2019.

\bibitem[Caron et~al.(2021)Caron, Touvron, Misra, J{\'e}gou, Mairal, Bojanowski, and Joulin]{caron2021emerging}
Caron, M., Touvron, H., Misra, I., J{\'e}gou, H., Mairal, J., Bojanowski, P., and Joulin, A.
\newblock Emerging properties in self-supervised vision transformers.
\newblock \emph{arXiv preprint arXiv:2104.14294}, 2021.

\bibitem[Chen et~al.(2020{\natexlab{a}})Chen, Kornblith, Norouzi, and Hinton]{chen2020simple}
Chen, T., Kornblith, S., Norouzi, M., and Hinton, G.
\newblock A simple framework for contrastive learning of visual representations.
\newblock \emph{preprint arXiv:2002.05709}, 2020{\natexlab{a}}.

\bibitem[Chen \& He(2021)Chen and He]{chen2021exploring}
Chen, X. and He, K.
\newblock Exploring simple siamese representation learning.
\newblock In \emph{Proceedings of the IEEE/CVF conference on computer vision and pattern recognition}, pp.\  15750--15758, 2021.

\bibitem[Chen et~al.(2020{\natexlab{b}})Chen, Fan, Girshick, and He]{chen2020mocov2}
Chen, X., Fan, H., Girshick, R., and He, K.
\newblock Improved baselines with momentum contrastive learning.
\newblock \emph{arXiv preprint arXiv:2003.04297}, 2020{\natexlab{b}}.

\bibitem[Chen et~al.(2021)Chen, Xie, and He]{chen2021empirical}
Chen, X., Xie, S., and He, K.
\newblock An empirical study of training self-supervised vision transformers.
\newblock \emph{arXiv preprint arXiv:2104.02057}, 2021.

\bibitem[Chu et~al.(2021)Chu, Tian, Zhang, Wang, Wei, Xia, and Shen]{chu2021conditional}
Chu, X., Tian, Z., Zhang, B., Wang, X., Wei, X., Xia, H., and Shen, C.
\newblock Conditional positional encodings for vision transformers.
\newblock \emph{arXiv preprint arXiv:2102.10882}, 2021.

\bibitem[Dosovitskiy et~al.(2014)Dosovitskiy, Springenberg, Riedmiller, and Brox]{Dosovitskiy2014DiscriminativeUF}
Dosovitskiy, A., Springenberg, J.~T., Riedmiller, M.~A., and Brox, T.
\newblock Discriminative unsupervised feature learning with convolutional neural networks.
\newblock In \emph{NIPS}, 2014.

\bibitem[Dosovitskiy et~al.(2020)Dosovitskiy, Beyer, Kolesnikov, Weissenborn, Zhai, Unterthiner, Dehghani, Minderer, Heigold, Gelly, et~al.]{dosovitskiy2020image}
Dosovitskiy, A., Beyer, L., Kolesnikov, A., Weissenborn, D., Zhai, X., Unterthiner, T., Dehghani, M., Minderer, M., Heigold, G., Gelly, S., et~al.
\newblock An image is worth 16x16 words: Transformers for image recognition at scale.
\newblock \emph{arXiv preprint arXiv:2010.11929}, 2020.

\bibitem[Dwibedi et~al.(2021)Dwibedi, Aytar, Tompson, Sermanet, and Zisserman]{dwibedi2021little}
Dwibedi, D., Aytar, Y., Tompson, J., Sermanet, P., and Zisserman, A.
\newblock With a little help from my friends: Nearest-neighbor contrastive learning of visual representations.
\newblock In \emph{Proceedings of the IEEE/CVF International Conference on Computer Vision}, pp.\  9588--9597, 2021.

\bibitem[Ermolov et~al.(2020)Ermolov, Siarohin, Sangineto, and Sebe]{Ermolov2020WhiteningFS}
Ermolov, A., Siarohin, A., Sangineto, E., and Sebe, N.
\newblock Whitening for self-supervised representation learning.
\newblock In \emph{International Conference on Machine Learning}, 2020.

\bibitem[Goyal et~al.(2021)Goyal, Duval, Reizenstein, Leavitt, Xu, Lefaudeux, Singh, Reis, Caron, Bojanowski, Joulin, and Misra]{goyal2021vissl}
Goyal, P., Duval, Q., Reizenstein, J., Leavitt, M., Xu, M., Lefaudeux, B., Singh, M., Reis, V., Caron, M., Bojanowski, P., Joulin, A., and Misra, I.
\newblock Vissl.
\newblock \url{https://github.com/facebookresearch/vissl}, 2021.

\bibitem[Grill et~al.(2020)Grill, Strub, Altch{\'e}, Tallec, Richemond, Buchatskaya, Doersch, Pires, Guo, Azar, et~al.]{grill2020bootstrap}
Grill, J.-B., Strub, F., Altch{\'e}, F., Tallec, C., Richemond, P.~H., Buchatskaya, E., Doersch, C., Pires, B.~A., Guo, Z.~D., Azar, M.~G., et~al.
\newblock Bootstrap your own latent: A new approach to self-supervised learning.
\newblock \emph{arXiv preprint arXiv:2006.07733}, 2020.

\bibitem[Hadsell et~al.(2006)Hadsell, Chopra, and LeCun]{Hadsell2006DimensionalityRB}
Hadsell, R., Chopra, S., and LeCun, Y.
\newblock Dimensionality reduction by learning an invariant mapping.
\newblock \emph{2006 IEEE Computer Society Conference on Computer Vision and Pattern Recognition (CVPR'06)}, 2:\penalty0 1735--1742, 2006.

\bibitem[He et~al.(2019)He, Fan, Wu, Xie, and Girshick]{he2019moco}
He, K., Fan, H., Wu, Y., Xie, S., and Girshick, R.
\newblock Momentum contrast for unsupervised visual representation learning.
\newblock \emph{arXiv preprint arXiv:1911.05722}, 2019.

\bibitem[He et~al.(2021)He, Chen, Xie, Li, Doll{\'a}r, and Girshick]{he2021masked}
He, K., Chen, X., Xie, S., Li, Y., Doll{\'a}r, P., and Girshick, R.
\newblock Masked autoencoders are scalable vision learners.
\newblock \emph{arXiv preprint arXiv:2111.06377}, 2021.

\bibitem[Hua et~al.(2021)Hua, Wang, Xue, Ren, Wang, and Zhao]{Hua_2021_ICCV}
Hua, T., Wang, W., Xue, Z., Ren, S., Wang, Y., and Zhao, H.
\newblock On feature decorrelation in self-supervised learning.
\newblock In \emph{Proceedings of the IEEE/CVF International Conference on Computer Vision (ICCV)}, pp.\  9598--9608, October 2021.

\bibitem[Jabri et~al.(2020)Jabri, Owens, and Efros]{jabri2020space}
Jabri, A., Owens, A., and Efros, A.
\newblock Space-time correspondence as a contrastive random walk.
\newblock \emph{Advances in neural information processing systems}, 33:\penalty0 19545--19560, 2020.

\bibitem[Johnson et~al.(2017)Johnson, Hariharan, Van Der~Maaten, Fei-Fei, Lawrence~Zitnick, and Girshick]{johnson2017clevr}
Johnson, J., Hariharan, B., Van Der~Maaten, L., Fei-Fei, L., Lawrence~Zitnick, C., and Girshick, R.
\newblock Clevr: A diagnostic dataset for compositional language and elementary visual reasoning.
\newblock In \emph{Proceedings of the IEEE conference on computer vision and pattern recognition}, pp.\  2901--2910, 2017.

\bibitem[Kingma \& Welling(2013)Kingma and Welling]{kingma2013auto}
Kingma, D.~P. and Welling, M.
\newblock Auto-encoding variational bayes.
\newblock \emph{arXiv preprint arXiv:1312.6114}, 2013.

\bibitem[Krizhevsky(2009)]{cifar10}
Krizhevsky, A.
\newblock Learning multiple layers of features from tiny images.
\newblock Technical report, 2009.

\bibitem[Krizhevsky et~al.(2009)Krizhevsky, Hinton, et~al.]{krizhevsky2009learning}
Krizhevsky, A., Hinton, G., et~al.
\newblock Learning multiple layers of features from tiny images.
\newblock 2009.

\bibitem[LeCun(2022)]{lecun2022path}
LeCun, Y.
\newblock A path towards autonomous machine intelligence version 0.9. 2, 2022-06-27.
\newblock 2022.

\bibitem[Liutkus et~al.(2021)Liutkus, C\'{\i}fka, Wu, Simsekli, Yang, and Richard]{pmlr-v139-liutkus21a}
Liutkus, A., C\'{\i}fka, O., Wu, S.-L., Simsekli, U., Yang, Y.-H., and Richard, G.
\newblock Relative positional encoding for transformers with linear complexity.
\newblock In Meila, M. and Zhang, T. (eds.), \emph{Proceedings of the 38th International Conference on Machine Learning}, volume 139 of \emph{Proceedings of Machine Learning Research}, pp.\  7067--7079. PMLR, 18--24 Jul 2021.
\newblock URL \url{https://proceedings.mlr.press/v139/liutkus21a.html}.

\bibitem[Misra \& van~der Maaten(2020)Misra and van~der Maaten]{misra2020self}
Misra, I. and van~der Maaten, L.
\newblock Self-supervised learning of pretext-invariant representations.
\newblock In \emph{Proceedings of the IEEE Conference on Computer Vision and Pattern Recognition}, pp.\  6707--6717, 2020.

\bibitem[Oord et~al.(2018)Oord, Li, and Vinyals]{oord2018representation}
Oord, A. v.~d., Li, Y., and Vinyals, O.
\newblock Representation learning with contrastive predictive coding.
\newblock \emph{arXiv preprint arXiv:1807.03748}, 2018.

\bibitem[Pathak et~al.(2016)Pathak, Krahenbuhl, Donahue, Darrell, and Efros]{pathak2016context}
Pathak, D., Krahenbuhl, P., Donahue, J., Darrell, T., and Efros, A.~A.
\newblock Context encoders: Feature learning by inpainting.
\newblock In \emph{Proceedings of the IEEE conference on computer vision and pattern recognition}, pp.\  2536--2544, 2016.

\bibitem[Pont-Tuset et~al.(2017)Pont-Tuset, Perazzi, Caelles, Arbel{\'a}ez, Sorkine-Hornung, and Van~Gool]{pont20172017}
Pont-Tuset, J., Perazzi, F., Caelles, S., Arbel{\'a}ez, P., Sorkine-Hornung, A., and Van~Gool, L.
\newblock The 2017 davis challenge on video object segmentation.
\newblock \emph{arXiv preprint arXiv:1704.00675}, 2017.

\bibitem[Press et~al.(2021)Press, Smith, and Lewis]{press2021train}
Press, O., Smith, N.~A., and Lewis, M.
\newblock Train short, test long: Attention with linear biases enables input length extrapolation.
\newblock \emph{arXiv preprint arXiv:2108.12409}, 2021.

\bibitem[Russakovsky et~al.(2015)Russakovsky, Deng, Su, Krause, Satheesh, Ma, Huang, Karpathy, Khosla, Bernstein, Berg, and Fei-Fei]{russakovsky2015imagenet}
Russakovsky, O., Deng, J., Su, H., Krause, J., Satheesh, S., Ma, S., Huang, Z., Karpathy, A., Khosla, A., Bernstein, M., Berg, A.~C., and Fei-Fei, L.
\newblock Imagenet large scale visual recognition challenge.
\newblock \emph{International Journal of Computer Vision}, 115\penalty0 (3):\penalty0 211--252, 2015.

\bibitem[Salakhutdinov \& Hinton(2007)Salakhutdinov and Hinton]{salakhutdinov2007learning}
Salakhutdinov, R. and Hinton, G.
\newblock Learning a nonlinear embedding by preserving class neighbourhood structure.
\newblock In \emph{Artificial Intelligence and Statistics}, pp.\  412--419. PMLR, 2007.

\bibitem[Su et~al.(2021)Su, Lu, Pan, Murtadha, Wen, and Liu]{su2021roformer}
Su, J., Lu, Y., Pan, S., Murtadha, A., Wen, B., and Liu, Y.
\newblock Roformer: Enhanced transformer with rotary position embedding.
\newblock \emph{arXiv preprint arXiv:2104.09864}, 2021.

\bibitem[Tian et~al.(2019)Tian, Krishnan, and Isola]{Tian2019ContrastiveMC}
Tian, Y., Krishnan, D., and Isola, P.
\newblock Contrastive multiview coding.
\newblock In \emph{European Conference on Computer Vision}, 2019.

\bibitem[Van~Horn et~al.(2018)Van~Horn, Mac~Aodha, Song, Cui, Sun, Shepard, Adam, Perona, and Belongie]{van2018inaturalist}
Van~Horn, G., Mac~Aodha, O., Song, Y., Cui, Y., Sun, C., Shepard, A., Adam, H., Perona, P., and Belongie, S.
\newblock The inaturalist species classification and detection dataset.
\newblock In \emph{Proceedings of the IEEE conference on computer vision and pattern recognition}, pp.\  8769--8778, 2018.

\bibitem[Vaswani et~al.(2017)Vaswani, Shazeer, Parmar, Uszkoreit, Jones, Gomez, Kaiser, and Polosukhin]{vaswani2017attention}
Vaswani, A., Shazeer, N., Parmar, N., Uszkoreit, J., Jones, L., Gomez, A.~N., Kaiser, {\L}., and Polosukhin, I.
\newblock Attention is all you need.
\newblock In \emph{Advances in neural information processing systems}, pp.\  5998--6008, 2017.

\bibitem[Vincent et~al.(2010)Vincent, Larochelle, Lajoie, Bengio, Manzagol, and Bottou]{vincent2010stacked}
Vincent, P., Larochelle, H., Lajoie, I., Bengio, Y., Manzagol, P.-A., and Bottou, L.
\newblock Stacked denoising autoencoders: Learning useful representations in a deep network with a local denoising criterion.
\newblock \emph{Journal of machine learning research}, 11\penalty0 (12), 2010.

\bibitem[Wu et~al.(2018)Wu, Xiong, Yu, and Lin]{wu2018unsupervised}
Wu, Z., Xiong, Y., Yu, S.~X., and Lin, D.
\newblock Unsupervised feature learning via non-parametric instance discrimination.
\newblock In \emph{Proceedings of the IEEE conference on computer vision and pattern recognition}, pp.\  3733--3742, 2018.

\bibitem[Xiao et~al.()Xiao, Wang, Efros, and Darrell]{xiaoshould}
Xiao, T., Wang, X., Efros, A.~A., and Darrell, T.
\newblock What should not be contrastive in contrastive learning.
\newblock In \emph{International Conference on Learning Representations}.

\bibitem[Xie et~al.(2021)Xie, Zhang, Cao, Lin, Bao, Yao, Dai, and Hu]{xie2021simmim}
Xie, Z., Zhang, Z., Cao, Y., Lin, Y., Bao, J., Yao, Z., Dai, Q., and Hu, H.
\newblock Simmim: A simple framework for masked image modeling.
\newblock \emph{arXiv preprint arXiv:2111.09886}, 2021.

\bibitem[Zbontar et~al.(2021)Zbontar, Jing, Misra, LeCun, and Deny]{zbontar2021barlow}
Zbontar, J., Jing, L., Misra, I., LeCun, Y., and Deny, S.
\newblock Barlow twins: Self-supervised learning via redundancy reduction.
\newblock \emph{arXiv preprint arXiv:2103.03230}, 2021.

\bibitem[Zhai et~al.(2019)Zhai, Puigcerver, Kolesnikov, Ruyssen, Riquelme, Lucic, Djolonga, Pinto, Neumann, Dosovitskiy, Beyer, Bachem, Tschannen, Michalski, Bousquet, Gelly, and Houlsby]{vtab}
Zhai, X., Puigcerver, J., Kolesnikov, A., Ruyssen, P., Riquelme, C., Lucic, M., Djolonga, J., Pinto, A.~S., Neumann, M., Dosovitskiy, A., Beyer, L., Bachem, O., Tschannen, M., Michalski, M., Bousquet, O., Gelly, S., and Houlsby, N.
\newblock A large-scale study of representation learning with the visual task adaptation benchmark, 2019.
\newblock URL \url{https://arxiv.org/abs/1910.04867}.

\bibitem[Zhou et~al.(2014{\natexlab{a}})Zhou, Lapedriza, Xiao, Torralba, and Oliva]{places205}
Zhou, B., Lapedriza, A., Xiao, J., Torralba, A., and Oliva, A.
\newblock Learning deep features for scene recognition using places database.
\newblock In Ghahramani, Z., Welling, M., Cortes, C., Lawrence, N., and Weinberger, K. (eds.), \emph{Advances in Neural Information Processing Systems}, volume~27. Curran Associates, Inc., 2014{\natexlab{a}}.
\newblock URL \url{https://proceedings.neurips.cc/paper/2014/file/3fe94a002317b5f9259f82690aeea4cd-Paper.pdf}.

\bibitem[Zhou et~al.(2014{\natexlab{b}})Zhou, Lapedriza, Xiao, Torralba, and Oliva]{zhou2014learning}
Zhou, B., Lapedriza, A., Xiao, J., Torralba, A., and Oliva, A.
\newblock Learning deep features for scene recognition using places database.
\newblock \emph{Advances in neural information processing systems}, 27, 2014{\natexlab{b}}.

\bibitem[Zhou et~al.(2021)Zhou, Wei, Wang, Shen, Xie, Yuille, and Kong]{zhou2021ibotyes}
Zhou, J., Wei, C., Wang, H., Shen, W., Xie, C., Yuille, A., and Kong, T.
\newblock Ibot: Image bert pre-training with online tokenizer.
\newblock \emph{arXiv preprint arXiv:2111.07832}, 2021.

\end{thebibliography}
\bibliographystyle{icml2024}
\newpage

\appendix
\onecolumn
\section*{Appendix}
\vspace{5mm}
\section{Noise collapse and weight tying}
\label{sup:proof_collapse}
Consider the following loss function where $n_j \sim N(0, \sigma I)$.: 
\begin{equation}
J = \Sigma_{i,j} \mathbb{E}_{n_j}[(F(An_j + \psi_j + \Tilde{m}, Bx_i) - y_j)^2]
\end{equation}
% Next, we show that $A = 0$ is a critical point of our loss function.
\begin{proposition}
If $A,B$ are different set of parameters then $\left. \frac{dJ}{dA} \right|_{A=0} = 0$
\end{proposition} 

\begin{proof}

\begin{align*}
\frac{\partial J}{\partial A} &= \sum_{i,j} \mathbb{E}_{n_j} [\frac{\partial}{\partial A} \|F(An_j + \psi_j + \Tilde{m}, Bx_i) - y_j\|^2] \\
&= \sum_{i,j} \mathbb{E}_{n_j} [2(F(An_j+\psi_j + \Tilde{m}, Bx_i) - y_j)\frac{\partial F(An_j+\psi_j + \Tilde{m}, Bx_i)}{\partial (An_j + \psi_j + \Tilde{m})} n^T_j]
\end{align*}

Set $A=0$, then derivative becomes:

\begin{align*}
\frac{\partial J}{\partial A}\Big|_{A=0} &= 2\sum_{i,j} (F(\psi_j + \Tilde{m}, Bx_i) - y_j)\frac{\partial F(\psi_j + \Tilde{m}, Bx_i)}{\partial (\psi_j + \Tilde{m})}\mathbb{E}_{n_j}[{n^T_j}] = 0
\end{align*}

\end{proof}

Define the following the loss with weight tying and the deterministic loss without noise: 

\begin{equation}
J_{tied}(A) = J(A,A) = \sum_{i,j} \mathbb{E}_{n_j}[(F(An_j + \psi_j + \Tilde{m}, Ax_i) - y_j)^2] \\
\end{equation}
\begin{equation}
J_{det}(B) = J(A=0,B) = \sum_{i,j} [(F(\psi_j + \Tilde{m}, Bx_i) - y_j)^2]
\end{equation}

\begin{proposition}
If $\left. \frac{dJ_{tied}}{dA} \right|_{A=0} = 0$ iff $\left. \frac{dJ_{det}(B)}{dB} \right|_{B=0} = 0$
\end{proposition} 

\begin{proof}

Next, we show that $A=0$ is a critical point of $J_{tied}$ iff $B=0$ is a critical point of $J_{det}$:
\begin{equation}
\frac{\partial J_{tied}}{\partial A}\Big|_{A=0} = \sum_{i,j} (F(\psi_j + \Tilde{m}, 0) - y_j)\nabla F(\psi_i, 0)x_i ^T 
\end{equation}
\begin{equation}
\frac{\partial J_{det}}{\partial B}\Big|_{B=0} = \sum_{i,j} (F(\psi_j + \Tilde{m}, 0) - y_j)\nabla F(\psi_j, 0)x_i ^T 
\end{equation}

Therefore $\frac{\partial J_{tie}}{\partial A}\Big|_{A=0} = 0$ iff $\frac{\partial J_{det}}{\partial B}\Big|_{B=0}$
\end{proof}

\section{Optimal Predictor}
\label{sup:optimal}
Consider a random variable $X$ (corresponding to the context in our case. For simplicity assume $X$ is just the positional embedding of the context) that is used to predict a variable $Y$ (corresponding to the target in our case). But now instead of predicting from $X$, we use a noise variable $Z$
that is independent of both $X,Y$, and provide the predictor with only the noisy result $R = g(X,Z)$. Here $g$ is some mixing function (in our case $g(x,z) = x+z$). We next derive the optimal predictor $f(R)$ in this case. Formally we want to minimize:
\begin{equation}
    E_{R,Y}[(f(R) - Y)^2]
\end{equation}
A classic result in estimation is that this is optimized by the conditional expectation $f(r) = E[Y|R=r]$. 

We simplify this as follows:
\begin{eqnarray*}
    E[Y|R=r] &=&  \sum_{x,y} y p(Y=y,X=x|R=r) \\
    &=& \sum_{x,y} y p(y|X=x)p(X=x|R=r) \\
    &=& \sum_x E[Y|X=x]p(X=x|R=r)
\end{eqnarray*}
where in the second line we used the fact that:
\begin{equation}
    p(y,x|r) = p(y|x,r)p(x|r) =  p(y|x)p(x|r)
\end{equation}
To further illustrate, consider the case where $z$ is Gaussian with zero mean and unit variance. Then $p(x|r)$ is also Gaussian with expectation $r$, and the expression above amounts to 
convolution of the clean expected values with a Gaussian:
\begin{equation}
    E[Y|R=r] = \int_x E[Y|X=x]\frac{1}{\sqrt{2\pi}}e^{-0.5(x-r)^2}dx
\end{equation}
\section{Experiments and Results}
\label{app:experiments}
We include the full implementation details, pretraining configs and evaluation protocols for the Ablations (see Appendix~\ref{app:ablations}), Downstream Tasks (Appendix~\ref{app:downstream}), as well as full results and comparisons to invariance-based methods.

\subsection{Ablations}
\label{app:ablations}
Here we pretrain all models for $300$ epochs using $4$ V100 nodes, on a total batch size of $2048$. In all the ablation study experiments, we follow the exact recipe of~\cite{assran2023self}. We include the full config in Table~\ref{tab:downstream_ablations} for completeness.

To evaluate the pretrained models, we use linear probing evaluation using 1\% of IN-1k~\citep{russakovsky2015imagenet}. To obtain the features of an image, we apply the target encoder over the image to obtain a sequence of tokens corresponding to the image. We then average the tokens to obtain a single representative vector. The linear classifier is trained over this representation, maintaining the rest of the target encoder layers fixed.

\subsection{Downstream Tasks}
\label{app:downstream}
Here we pretrain I-JEPA with StoP for $600$ epochs using $4$ V100 nodes, on a total batch size of $2048$ using ViT-B (see config in Table~\ref{tab:downstream_base}) and ViT-L (see config in Table~\ref{tab:downstream_large}). For ViT-H we use float16 and train for $300$ epochs and follow the config in Table~\ref{tab:downstream_huge}. We follow similar configs compared to~\cite{assran2023self} except we usually use a lower learning rate. Intuitively, since StoP is stochastic it is more sensitive to high learning rates. 

For evaluation on downstream tasks, we use the features learned by the target-encoder and follow the protocol of VISSL~\cite{goyal2021vissl} that was utilized by I-JEPA~\cite{assran2023self}. Specifically, we report the best linear evaluation number among the average-pooled patch representation of the last layer and the concatenation of the last $4$ layers of the average-pooled patch representations. We report full results including comparisons to invariance-based methods for IN-1k linear evaluation Table~\ref{tb:supp:lineareval}, 1\% IN-1k finetuning results in~Table~\ref{supp:tab:finetune}, and other downstream tasks in Table~\ref{supp:tb:transfer-classification}.

For baselines that use Vision Transformers~\cite{dosovitskiy2020image} with a {\tt [cls]} token (e.g, iBOT~\cite{zhou2021ibotyes}, DINO~\cite{caron2021emerging} or MAE~\cite{he2021masked}), we use the default configurations of VISSL~\cite{goyal2021vissl} to evaluate the publicly available checkpoints on iNaturalist18~\cite{van2018inaturalist}, CIFAR100~\cite{krizhevsky2009learning}, Clevr/Count~\cite{johnson2017clevr, vtab}, Clevr/Dist~\cite{johnson2017clevr, vtab}, and Places205~\cite{zhou2014learning}.
Following the evaluation protocol of VISSL~\cite{goyal2021vissl}, we freeze the encoder and return the best number among the {\tt [cls]} token representation of the last layer and the concatenation of the last $4$ layers of the {\tt [cls]} token.

For semi-supervised video object segmentation, we propagate the first labeled frame in a video using the similarity between adjacent frames features. To label the video using the frozen features, we follow the code and hyperparams of~\cite{caron2021emerging}. To evaluate the segmented videos, we use the evaluation code of DAVIS 2017~\citep{pont20172017} and include full results in Table~\ref{tb:supp:labelprop}.

\begin{table}[H]
\vspace{5mm}
\begin{minipage}{0.45\textwidth}
\tablestyle{6pt}{1.02}
\footnotesize
\begin{tabular}{y{96}|y{68}}
config & value \\
\shline
optimizer & AdamW \\%\cite{Loshchilov2019} \\
epochs & 300 \\
learning rate & $1e^{-3}$ \\
weight decay &  $(0.04, 0.4)$ \\
batch size & 2048 \\
learning rate schedule & cosine decay \\%\cite{Loshchilov2016} \\
warmup epochs & 15\\
encoder arch. & ViT-B  \\
predicted targets & 4 \\
predictor depth & 6  \\
predictor attention heads & 12 \\ 
predictor embedding dim. & 384 \\
$\sigma$ (noise hyperparam) & $0.25$ \\

\end{tabular}
\vspace{-.5em}
\caption{\small\textbf{Pretraining setting for ablations}. Using ViT-B encoder, trained for $300$ epochs, config strictly follows~\cite{assran2023self}.}
\label{tab:downstream_ablations} \vspace{-.5em}
\end{minipage}%
\hfill
\begin{minipage}{0.45\textwidth}
\tablestyle{6pt}{1.02}
\footnotesize
\begin{tabular}{y{96}|y{68}}
config & value \\
\shline
optimizer & AdamW \\%\cite{Loshchilov2019} \\
epochs & $600$ \\
learning rate & $8e^{-4}$ \\
weight decay &  $(0.04, 0.4)$ \\
batch size & $2048$ \\
learning rate schedule & cosine decay \\% \cite{Loshchilov2016} \\
warmup epochs & 15\\
encoder arch. & ViT-B  \\
predicted targets & 4 \\
predictor depth & 6  \\
predictor attention heads & 12 \\ 
predictor embedding dim. & 384 \\
$\sigma$ (noise hyperparam) & $0.25$ \\
\end{tabular}
\vspace{-.5em}
\caption{\small\textbf{Pretraining setting for downstream tasks (ViT-B)}. All models trained for  $600$ epochs.}
\label{tab:downstream_base} \vspace{-.5em}
\end{minipage}
\end{table}

\begin{table}[H]
\vspace{5mm}
\begin{minipage}{0.45\textwidth}
\tablestyle{6pt}{1.02}
\footnotesize
\begin{tabular}{y{96}|y{68}}
config & value \\
\shline
optimizer & AdamW \\%\cite{Loshchilov2019} \\
epochs & $600$ \\
learning rate & $8e^{-4}$ \\
weight decay &  $(0.04, 0.4)$ \\
batch size & $2048$ \\
learning rate schedule & cosine decay \\%\cite{Loshchilov2016} \\
warmup epochs & 15\\
encoder arch. & ViT-L  \\
predicted targets & 4 \\
predictor depth & 12  \\
predictor attention heads & 16 \\ 
predictor embedding dim. & 384 \\
$\sigma$ (noise hyperparam) & $0.25$ \\
\end{tabular}
\vspace{-.5em}
\caption{\small\textbf{Pretraining setting for downstream tasks (ViT-L)}. All models trained for  $600$ epochs.}
\label{tab:downstream_large} \vspace{-.5em}
\end{minipage}%
\hfill
\begin{minipage}{0.49\textwidth}
\tablestyle{6pt}{1.02}
\footnotesize
\begin{tabular}{y{96}|y{68}}
config & value \\
\shline
optimizer & AdamW \\%\cite{Loshchilov2019} \\
epochs & $600$ \\
learning rate & $1e^{-3}$ \\
weight decay &  $(0.04, 0.4)$ \\
batch size & $2048$ \\
learning rate schedule & cosine decay \\%\cite{Loshchilov2016} \\
warmup epochs & 40\\
encoder arch. & ViT-H  \\
predicted targets & 4 \\
predictor depth & 12  \\
predictor attention heads & 16 \\ 
predictor embedding dim. & 384 \\
$\sigma$ (noise hyperparam) & $0.2$ \\
\end{tabular}
\vspace{-.5em}
\caption{\small\textbf{Pretraining setting for downstream tasks (ViT-H)}. Trained for  $300$ epochs.}
\label{tab:downstream_huge} \vspace{-.5em}
\end{minipage}
\end{table}

\begin{table*}[t]
    \centering
    \small
    \begin{tabular}{llcccccc}
            \bf\small Method & \bf\small Arch. & \bf\small CIFAR100 & \bf\small Places205 &  \bf\small iNat18 & \bf\small CLEVR/Count & \bf\small CLEVR/Dist  \\
        \toprule
        \multicolumn{5}{l}{\small\bf\it Invariance-based methods \textbf{(use extra image augmentations)}}\\
        \small DINO & \small ViT-B/16 & 84.8 & 55.2 & 50.1 & 83.2 & 53.4\\ [1ex]
        \multirow{2}{*}{\small iBOT} & \small ViT-B/16 & 85.5 & 56.7 & 50.0 & 62.1 & 64.6 \\
        & \small ViT-L/16 & 88.3 & 60.4 & 57.3 & 85.7 & 62.8 \\
        \midrule
        \multicolumn{5}{l}{\small\bf\it Masked Image Modeling Methods}\\
        
        \small data2vec & \small ViT-L/16 & 81.6 & 54.6 & 28.1 & 85.3 & 71.3 \\[1ex]
        \multirow{3}{*}{\small MAE} & \small ViT-B/16 & 68.1 & 49.2 & 26.8 & 86.6 & 70.8 \\
         & \small ViT-L/16 & 77.4 & 54.4 & 33.0 & \textbf{92.1} & 73.0 \\
          & \small ViT-H/14 & 77.3 & 55.0 & 32.9 & 90.5 & 72.4 \\
         [1ex]

        \multirow{3}{*}{\small I-JEPA} & \small ViT-B/16 & 69.2 & 53.4 & 43.4 & 82.2 & 70.7 \\
        & \small ViT-L/16 & 83.6 & 56.5 & 48.4 & 85.6 & 71.2\\
        & \small ViT-H/14 & 87.5 & 58.4 & 47.6 & 86.7 & 72.4 \\[1ex]

        \multirow{3}{*}{\small +StoP} & \small \cc ViT-B/16 & \cc 81.2 & \cc 54.3 & \cc 44.7 & \cc 83.7 & \cc 71.3\\
        & \small \cc ViT-L/16 & \cc 84.7 & \cc 57.2 & \cc 49.2 & \cc 85.7 & \cc 70.2 \\
        
        & \small \cc ViT-H/14 & \cc \bf 87.7 & \cc \bf \textbf{58.4} & \cc \bf 50.9 & \cc 88.0 & \cc \bf 72.5 \\
        \bottomrule
    \end{tabular}
    \caption{\small
    \textbf{Linear-probe transfer for various downstream tasks}. Linear-evaluation on downstream image classification, object counting, and tracking tasks. StoP significantly outperforms previous MIM methods that don't utilize image augmentations like I-JEPA and MAE, and decreases the gap with the best invariance-based methods that utilize data augmentations during pretraining.}
    \label{supp:tb:transfer-classification}
\end{table*}

\begin{table}[t]
  \centering
  \begin{minipage}{0.49\linewidth}
    \begin{minipage}{1\linewidth}
\resizebox{0.95\linewidth}{!}{
\begin{tabular}{lllc}
        \bf\small Method & \bf\small Arch. & \bf\small Epochs & \bf\small Top-1\\
        \toprule
        \multicolumn{4}{l}{\small\bf\it Invariance-based methods} (\textbf{\small use extra image augmentations})\\
        \small SimCLR v2 & \small RN152 ($2\times$) & 800 & 79.1 \\ [1ex]
        
        \small BYOL & \small RN200 ($2\times$) & 800 & 79.6 \\ [1ex]

        \multirow{2}{*}{\small DINO} & \small ViT-B/16 & 400 & 78.1 \\
        & \small ViT-B/8 & 300 & 80.1 \\ [1ex]
        
        \multirow{2}{*}{\small MoCo v3} & \small ViT-B/16 & 300 & 76.7\\ 
         & \small ViT-BN-L/7 & 300 & 81.0\\[1ex]
        
        \small MSN & \small ViT-L/7 & 200 & 80.7 \\ [1ex]
        
        \multirow{2}{*}{\small iBOT} & \small ViT-B/16 & 250 & 79.8\\
        & \small ViT-L/16 & 250 & 81.0 \\
        \midrule
        \multicolumn{4}{l}{\small\bf\it Masked Image Modeling methods}\\
        data2vec & \small ViT-L/16 & 1600 & 77.3 \\[1ex]
        \multirow{2}{*}{\small MAE} & \small ViT-B/16 & 1600 & 68.0\\
        & \small ViT-L/16 & 1600 & 75.8\\
        & \small ViT-H/14 & 1600 & 77.2\\
        [1ex]
        \multirow{2}{*}{\small I-JEPA} 
        & \small ViT-B/16 & 600 & {72.9}\\
        & \small ViT-L/16 & 600  & {77.5}\\ 
        & \small ViT-H/14 &  300 & {79.3}\\[1ex]
        \multirow{2}{*}{+StoP (ours)} 
        & \cc\small ViT-B/16 & \cc 600 & \cc {74.5}\\
        & \cc\small ViT-L/16 & \cc 600  & \cc{78.5}\\
        & \cc\small ViT-H/14 &  \cc 300 & \cc\bf{79.6}\\
        \bottomrule
    \end{tabular}}
\caption{\small{\bf Linear-evaluation on IN-1k}. Performance of invariance based and MIM approaches.}
\label{tb:supp:lineareval}
  \end{minipage}
  \vfill

  \end{minipage}%
  \hfill
\begin{minipage}{0.45\linewidth}
    \begin{minipage}{1\linewidth}
    \centering
    \small
    \resizebox{\linewidth}{!}{%
    \begin{tabular}{l l c c c }
        \bf\small Method & \bf\small Arch. & \bf\small J-Mean & \bf\small F-Mean & \bf\small J\&F Mean\\
        \toprule
        \multicolumn{5}{l}{\small\bf\it Invariance-based methods \textbf{(use extra image augmentations)}}\\
        \small DINO & \small ViT-B/16 & 60.7 & 63.9 & 62.3\\ [1ex]
        \multirow{2}{*}{\small iBOT} & \small ViT-B/16 & 60.9 & 63.3 & 62.1\\
        & \small ViT-L/16 &  61.7 & 63.9 & 62.8 \\
        \midrule
        \multicolumn{5}{l}{\small\bf\it Masked Image Modeling Methods}\\
        \multirow{3}{*}{\small MAE} & \small ViT-B/16 & 49.4 & 52.6 & 50.9 \\
        & \small ViT-L/16 & 52.5 & 54.3 & 53.4\\
        & \small ViT-H/14 & 54.0 & 57.0 & 55.5 \\ [1ex]
        \multirow{2}{*}{\small I-JEPA} & \small ViT-B/16 &  56.1 & 56.2 & 56.1 \\
        & \small ViT-L/16 & 56.1 & 55.7 & 55.9 \\ 
        & \small ViT-H/14 & 58.5 & 60.9 & 59.7 \\ [1ex]
        \multirow{3}{*}{\small +StoP} & \cc\small ViT-B/16 & \cc 56.6 & \cc 57.3 & \cc 57.0 \\
        & \cc\small ViT-L/16 & \cc 58.1 & \cc 58.7 & \cc 58.4\\
        & \cc\small ViT-H/14 & \cc \bf 58.9 & \cc \bf 61.2 & \cc \bf 60.1\\
        \bottomrule
\end{tabular}}
    \caption{\small{\textbf{Video objects semi-supervised segmentation.} MIM and Invarianced-based methods. Results reported on DAVIS 2017 dataset. }}
  \label{tb:supp:labelprop}
  \end{minipage}%
  \vfill
  \vspace{21mm}
  \begin{minipage}{1\linewidth}
      \resizebox{\linewidth}{!}{%
  \begin{tabular}{lllc}
        \bf\small Method & \bf\small Arch. & \bf\small Epochs & \bf\small Top-1\\
        \toprule
        \multicolumn{4}{l}{\small\bf\it Invariance-based methods} (\textbf{\small use extra image augmentations})\\
        \small DINO & \small ViT-B/8 & 300 & 70.0 \\
        \multirow{1}{*}{\small iBOT} & \small ViT-B/16 & 400 & 69.7\\
        \midrule
        \multicolumn{4}{l}{\small\bf\it Masked Image Modeling methods}\\
        \multirow{1}{*}{\small MAE} & \small ViT-L/16 & 1600 & 67.0\\
        \multirow{1}{*}{\small I-JEPA} 
        & \small ViT-L/16 & 600  & {69.4}\\ 
        \multirow{1}{*}{+StoP (ours)} 
        & \cc\small ViT-L/16 & \cc 600  & \cc{\textbf{71.7}}\\
        \bottomrule
    \end{tabular}
    }
  \caption{\textbf{Finetuning results over ImageNet with 1\% labels.} Comparison of MIM and invariance-based methods.}
  % \vspace{-1.5em}
    \label{supp:tab:finetune}
\end{minipage}
  \end{minipage}
\end{table}

\end{document}